\documentclass[10pt,twocolumn,letterpaper]{article}

\usepackage{wacv}      %
\def\wacvPaperID{672} %
\def\confName{WACV}
\def\confYear{2025}

\usepackage{microtype}
\usepackage{graphicx}
\usepackage{booktabs} %
\usepackage{multirow}
\usepackage{math}

\usepackage[numbers,sort&compress]{natbib}

\usepackage[pagebackref,breaklinks,colorlinks]{hyperref}

\usepackage{amsmath}
\usepackage{amssymb}
\usepackage{mathtools}
\usepackage{amsthm}
\usepackage[capitalize,nameinlink]{cleveref}
\usepackage{nicefrac}
\usepackage{bm}
\usepackage[]{fontawesome5}
\usepackage[]{safecolours}
\usepackage{enumitem}
\usepackage{colortbl}
\usepackage{arydshln}

\crefname{section}{Sec.}{Secs.}
\crefname{proposition}{Prop.}{Props.}
\crefname{model}{Mod.}{Mods.}
\crefname{appendix}{App.}{Apps.}
\crefname{algorithm}{Alg.}{Algs.}
\crefname{prop}{Prop.}{Props.}

\usepackage{xcolor}
\colorlet{asblue}{colour1}
\colorlet{asgreen}{colour2}
\colorlet{asred}{colour3}
\colorlet{aspurple}{colour4}
\colorlet{asyellow}{colour5}
\colorlet{asturquoise}{colour6}
\colorlet{asgray}{colour7}

\newcolumntype{a}{>{\columncolor{asgray!20}}c}

\usepackage{tikz, pgfplots}
\pgfplotsset{compat=1.18}
\usetikzlibrary{patterns,shapes.arrows,calc,arrows.meta}
\usetikzlibrary{shapes.geometric}

\definecolor{crimson2143940}{RGB}{214,39,40}
\definecolor{darkgray176}{RGB}{176,176,176}
\definecolor{darkorange25512714}{RGB}{255,127,14}
\definecolor{forestgreen4416044}{RGB}{44,160,44}
\definecolor{steelblue31119180}{RGB}{31,119,180}
	
\definecolor{ruiblue}{HTML}{4580DC}
\definecolor{ruired}{HTML}{D4504E}
\definecolor{ruigreen}{HTML}{83CD80}
\definecolor{ruiorange}{HTML}{FFC66C}

\usetikzlibrary{quotes, arrows.meta}

\usetikzlibrary{positioning}
\usetikzlibrary{3d} %

\newdimen\datasize
\newcommand{\dataset}[1]{%
\begin{tikzpicture}[baseline=-1ex,inner sep=0,outer sep=0]
\foreach \x [count=\i] in {1,2,3,4,5,6} 
  \node[minimum width=\datasize,minimum height=\datasize,rounded corners=1pt,path picture={\node at (path picture bounding box.center){\includegraphics[height=\datasize]{fig/data/#1/img\i}};}] at (1.05*\i*\datasize,0){};
\end{tikzpicture}
}

\newcommand{\addspace}{\,}
\newcommand{\rdcspace}{\!}

\renewcommand{\paragraph}[1]{\smallskip\noindent\textbf{#1}~~}
\newcommand{\mbd}[1]{\bm{#1}}
\DeclareRobustCommand{\bbone}[0]{\tikz[baseline=-1ex]\node[rotate=-30, inner sep=0pt, outer sep=0pt]{\faBone};}
\newcommand{\vx}{\mbf{x}}

\newcommand{\vtheta}{\bm{\theta}}
\newcommand{\vpsi}{\bm{\psi}}
\newcommand{\vphi}{\bm{\phi}}
\newcommand{\vepsilon}{\bm{\epsilon}}

\DeclareMathOperator*{\argmin}{arg\,min}

\newcommand{\Div}{\mathbf{Div}}
\newcommand{\E}{\mathcal{E}}
\newcommand{\Ehat}{\hat{\mathcal{E}}}
\newcommand{\Esrc}{\E(\vtheta;\domainD)}
\newcommand{\Ehatsrc}{\Ehat_{\text{ERM}}(\vtheta;\mcD)}
\newcommand{\Ehatsrcsam}{\Ehat_{\text{SAM}}(\vtheta;\mcD)}
\newcommand{\Etgt}{\E(\vtheta; \domainT)}
\newcommand{\Etgti}{\E(\vtheta;\domainT^{(i)})}

\newcommand{\domainD}{\mathfrak{D}}
\newcommand{\domainT}{\mathfrak{T}}

\newcommand{\loss}{\ell}
\newcommand{\expect}{\mathbb{E}}
\newcommand{\defeq}{\stackrel{\Delta}{=}}
\newcommand{\consta}{\sqrt{\frac{(v_k[\ln (N / v_k)+1]+\ln (K / \delta))}{2N}}}
\newcommand{\constb}{\sqrt{\frac{v \ln (N / v)+\ln (2 / \delta)}{N}}}

\usepackage{xspace}
\newcommand*\circled[1]{\tikz[baseline=(char.base)]{\node[circle, draw=black!30, black!80, inner sep=1.1pt, fill=white, line width=1pt,font=\small\bf] (char) {\textcolor{black} #1};}}

\theoremstyle{plain}
\newtheorem{theorem}{Theorem}[section]

\newtheorem{lemma}[theorem]{Lemma}

\theoremstyle{definition}

\theoremstyle{remark}

\begin{document}

\title{Flatness Improves Backbone Generalisation in Few-shot Classification}
\author{
  Rui Li$^{1}$ \qquad Martin Trapp$^{1}$  \qquad Marcus Klasson$^{1,2}$ \qquad Arno Solin$^{1,2}$ \\
~\\
\small{
$^1$Aalto University\quad
$^2$Finnish Center for Artificial Intelligence \quad
}
}
\maketitle

\begin{abstract}

	Deployment of deep neural networks in real-world settings typically requires adaptation to new tasks with few examples.
	Few-shot classification (FSC) provides a solution to this problem by leveraging pre-trained backbones for fast adaptation to new classes.
	However, approaches for multi-domain FSC typically result in complex pipelines aimed at information fusion and task-specific adaptation without consideration of the importance of backbone training.
	In this work, we introduce an effective strategy for backbone training and selection in multi-domain FSC by utilizing flatness-aware training and fine-tuning.
	Our work is theoretically grounded and empirically performs on par or better than state-of-the-art methods despite being simpler.
	Further, our results indicate that backbone training is crucial for good generalisation in FSC across different adaptation methods. \looseness-1

\end{abstract}

\section{Introduction}
Deep neural networks have shown remarkable successes when trained on large labelled data sets. 
However, in many real-world applications, access to labelled data is limited and, therefore, training a network with good generalisation behaviour is challenging.
This has sparked research on methods to adapt pre-trained models to new data domains and concepts, \ie, classes, even if only a few examples exist. 
Few-shot classification (FSC) methods \citep{chen2019closer,wang2020generalizing} have shown promising performance in these scenarios.
Earlier works on FSC \citep{vinyals2016miniimagenet,finn2017model,snell2017prototypical} focus on homogeneous/single-domain learning tasks (\eg, \cite{lake2011omniglot,vinyals2016miniimagenet}), \ie, training data and test data both come from the same domain. 
However, as later shown, these elaborated methods can often be surpassed by simple fine-tuning \citep{chen2019closer, tian2020rethinking, dhillon2020baseline} when the distribution shift between training and test data is sufficiently small. 
Consequently, Meta-Dataset \citep{metadataset} was introduced as a heterogeneous, multi-domain benchmark to reflect more realistic settings in which models must be adapted to previously unseen data domains with potentially large distribution shifts. 

\begin{figure}[t]
	\centering
	\includegraphics{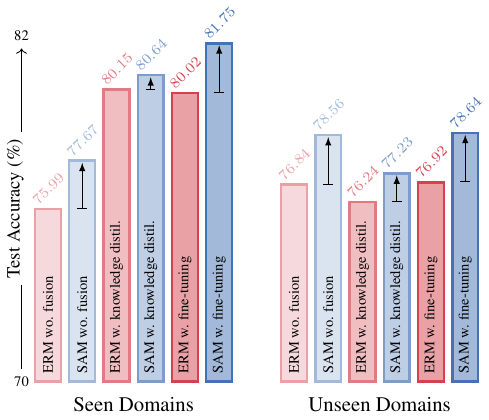}

	\caption{Average test accuracy on the Meta-Dataset benchmark for different backbone trainings using the adaptation by \cite{li2022tsa}. Across different information fusion methods, sharpness-aware minimisation (\textcolor{asblue}{SAM}) leads to better performance than empirical risk minimisation (\textcolor{asred}{ERM}), showing flatness improves backbone generalisation.
	}
	\label{fig:fig1}
\end{figure}

\message{The column width is: \the\columnwidth}

Existing approaches to tackle the multi-domain FSC problem can roughly be grouped into three categories: {\em (i)}~learn to fuse information of independent backbones to obtain generalisable features \cite{dvornik2020sur,liu2020urt}, 
{\em (ii)}~learn an auxiliary network that predicts parameters of task-specific layers added to the backbone  \cite{requeima2019cnaps,bateni2020simplecnaps,bateni2022transductive}, or 
{\em (iii)}~directly learn the parameters of task-specific layers during adaptation  \cite{triantafillou2021flute,li2021url,li2022tsa}.
Crucially, all of those approaches heavily depend on the generalisation behaviour of the backbone(s) to be transferable to new domains and concepts.
However, investigating effective backbone training with good generalisation behaviour is an overlooked topic and still in its infancy.

Recently, the connection between model generalisation and flat optimum in the loss landscape has been studied empirically and theoretically \cite{keskar2016flatness,dziugaite2017flatness,jiang2019flatness} in the deep learning community.
To this end, optimisers that seek flat minima have been proposed \cite{foret2020sam,izmailov2018swa} and have shown to improve generalisation in various deep learning applications \cite{chen2021sam-app1,bahri2022sam-app2}.
Moreover, recent findings in domain adaptation \cite{cha2021swad} indicate that flatness can improve generalisation in domain generalisation settings.
Raising the question to what extend does flatness improve generalisation in FSC.
An additional challenge introduced in multi-domain settings is how to avoid domain conflicts when fusing information. 
A simple approach that has shown to be effective is to train one backbone per source domain rather than training one general backbone \cite{dvornik2020sur}. 
The information from each backbone can then be fused to obtain multi-domain feature representations that generalise to new tasks.
However, as the number of training samples from each domain can vary (\eg, in Meta-dataset), effective fusing of the information from different backbones can be challenging.
In this work, we introduce a theoretically justified and effective approach for backbone training and selection in multi-domain FSC. 
The approach is based on {\em (i)}~seeking flat solutions during backbone training (\eg, \cite{foret2020sam,mollenhoff2022bsam}) to improve generalization, {\em (ii)}~fusing information in the multi-domain setting using fine-tuning, {\em (iii)}~and selecting the most compatible backbone for new tasks in cross-domain FSC settings.
As shown in \cref{fig:fig1}, combining these seemingly simple strategies results in a competitive approach compared to the state-of-the-art methods without changes of the adaptation method.
Moreover, we observe that sharpness-aware training (SAM) of the backbone consistently improves generalisation in FSC over standard empirical risk-minimization (ERM).
We present theoretical and empirical findings indicating that careful backbone training is crucial in FSC.
Henceforth, we advocate for more careful treatments of the used backbones and a more competitive baseline. %

Our contributions can be summarized as follows: 
\begin{itemize}%
	
	\item We introduce an effective approach for multi-domain FSC which performs on par or better than state-of-the-art methods despite being simpler.
		
	\item We present theoretical results that flatness can improve generalisation in FSC, motivating our approach to backbone training and selection.

	\item We show empirical evidence that {\em (i)}~flatness helps generalisation in FSC, {\em (ii)}~fine-tuning is an effective information fusing method, and {\em (iii)}~combining flatness and fine-tuning in the backbone training results in better performance compared to the state-of-the-art.

\end{itemize}

\begin{figure}[!t]
	\centering\footnotesize
	\input{fig/tikz/flatness.tex}
	\caption{Illustration that solutions in flat areas on the training loss can result in better generalisation behaviour on the test loss. }
	\label{fig:flatness}
\end{figure}

\section{Background}
\label{sec:background}
We use calligraphic letters to denote sets (\eg, the query set $\mcQ$), denote domains using fraktur font (\eg, the target domain $\domainT$), and use bold letters for vectors.
Further, we denote the risk associated with a hypothesis/model $f_{\vtheta}(\cdot)$ as $\mcE$ and use the hat-symbol $\hat{\mcE}$ whenever the risk is calculated \wrt an empirical distribution.

\subsection{Few-shot Classification} 

We assume to be given a 
training set $\mcD= \{(\vx_n, y_n)\}_{n=1}^N$ of $N$ input--output pairs, where $\vx$ denotes the input and $y$ its corresponding class label.
In FSC, the goal is to learn a model $f_{\vtheta}(\cdot)$ that can adapt to new classes or domains from few examples.
At test time, we are given multiple tasks $t = 1, \dots, T$ consisting of support $\mcS_t=\{(\vx_i, y_i)\}_{i=1}^{|\mcS_t|}$ and query sets $\mcQ_t=\{(\vx_j, y_j)\}_{j=1}^{|\mcQ_t|}$ with $\mcQ_t \cap \mcS_\tau = \emptyset$ and $|\mcS_t| \ll N$. 
Support sets and query sets in FSC are similar to training data and test data respectively in supervised learning.
Note the sets of classes at training and test time are disjoint and we might have a domain shift.

Let the model parameters be given as $\vtheta = \{\vphi, \vpsi\}$, where we refer to $\vphi$ as the task-agnostic backbone and $\vpsi$ as task-specific parameters.
Task-agnostic parameters are learnt on the training set and task-specific parameters are learnt on the support set during evaluation.
Specifically, learning a model in FSC can be divided into three phases: 
{\em (i)}~during training, learning $\vphi$ on the training set,
{\em (ii)}~at test time, injecting task-specific layers parametrised by $\vpsi$, and
{\em (iii)}~learning $\vpsi$ during the adaptation on the support set of a sampled task while keeping $\vphi$ fixed.
If the tasks are sampled from the same domain as the training set, we refer to the setting as in-domain and as cross-domain otherwise.
If multiple training sets are available, \eg, in multi-domain setting, we may have a backbone per data set or learn a general backbone. \looseness-1

\subsection{Sharpness-aware Minimisation}

Given a data set $\mcD$, the empirical risk minimisation (ERM) problem for a model $f_{\vtheta}(\cdot)$ and a pointwise loss function $\ell(\cdot, \cdot)$ is defined as:
\begin{equation}
   \arg\min_\theta \hat{\mcE}_{\text{ERM}} (\vtheta; \mcD, \alpha) = \arg\min_\theta \mathcal{L}(\vtheta) + \frac{\alpha}{2} \|\vtheta\|^2,
\end{equation}
where $\mathcal{L}(\vtheta)=\frac1N \sum_{n=1}^N \ell(f_{\vtheta}(\vx_n), y_n)$ and $\frac{\alpha}{2} \|\vtheta\|^2$ is a regularization term.
The goal of sharpness-aware minimisation (SAM) \cite{foret2020sam} is to reduce the generalisation error by additionally accounting for the loss geometry in the ERM objective.
In particular, SAM aims to simultaneously minimise the loss value and the loss sharpness by seeking parameters whose entire neighbourhood have uniformly low loss values under some $\vepsilon$ perturbation with $\|\vepsilon\| \leq \rho$, \ie,
\begin{equation}\label{eq:sam_obj}
	\hat{\mcE}_{\text{SAM}} (\vtheta; \mcD, \rho, \alpha) = \max_{\|\vepsilon\| \leq \rho} \mathcal{L}(\vtheta + \vepsilon) + \frac{\alpha}{2} \|\vtheta\|^2,
\end{equation}
where $\rho > 0$ defines the radius of the neighbourhood of $\vtheta$.

\cref{fig:flatness} provides an intuition that solutions to the SAM objective can result in improved generalisation behaviour of the trained model. 
Optimising the ERM objective can result in lower training loss compared to SAM.
However, in the case of a distribution shift between the training and test sets, a solution in a flat loss landscape may yield better generalisation on the test loss.
Based on this intuition, we study the use of SAM as a replacement of ERM in FSC.

\section{Methods}\label{sec:method}
We will now study the link between flatness and generalisation in the FSC in \cref{sec:method:theory}. 
Then in \cref{sec:method:backbone}, we introduce our backbone training protocol: based on our theoretical results, we use a flatness-seeking objective for backbone training and introduce a backbone selection method to choose the best backbone for adaptation. 
For information fusion, we propose to use a %
fine-tuning strategy to fuse information in the multi-domain settings. 

\begin{figure}
	\centering
	\includegraphics[width=\columnwidth]{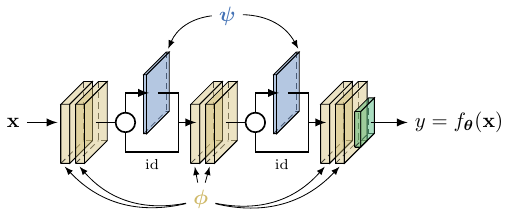}
	\vspace*{-2em}
	\caption{Decomposition of $f_{\vtheta}(\cdot)$ with $\vtheta = \{\vphi, \vpsi\}$ into task-agnostic layers parametrised by $\textcolor{asyellow}\vphi$ and task-specific layers parametrised by $\textcolor{asblue}\vpsi$. Additional gates $\circ$ are used to switch between task-specific layers and the identity function. Note that this construction is only for theoretical purposes and does not imply any additional operations in practice.}
	\label{fig:nnfunction}
\end{figure}
\subsection{Flatness Leads to a Better Backbone for Adaptation}\label{sec:method:theory}
In FSC the backbone is trained using data from the source domain $\domainD$ and later evaluated on data from the target domain $\domainT$.
We assume test tasks are sampled independently from sub-domains of the target domain $\domainT_t \subset \domainT$.
Recall that the source and target (sub)-domains have disjoint sets of classes and a possible distribution shift.

Let $\ell(\cdot, \cdot)$ be a bounded loss function where $\loss(y_1, y_2)=0$ iff $y_1=y_2$. 
Given a model $f_{\vtheta}(\cdot)$, we denote the \emph{empirical} risk on the source domain as $\Ehatsrc$ where $\mcD \in \domainD$, the SAM loss on $\domainD$ as $\Ehatsrcsam$, and the risk on $\domainT$ as:
\begin{equation}
	\Etgt \defeq \expect_{\domainT_t \sim \domainT} \left[\expect_{(\vx, y) \sim \domainT_t}\left[\loss(f_{\vtheta}(\vx), y)\right] \right]. 
\end{equation}
In the setting of domain generalisation, \cite{cha2021swad} showed the target domain loss can be bounded by the SAM loss on the source domain, the divergence between the source and the target, and a confidence bound that depends on the hyperparameter $\rho$ of the SAM loss.

\begin{theorem}[\cite{cha2021swad}]\label{thm:swad-theorem}
First, let $\{\Theta_k \subset \mathbb{R}^d, k=1, \ldots, K\}$, where $d$ is dimension of $\Theta$, be a finite cover of the parameter space $\Theta$ consisting of $K$ closed balls with radius $\nicefrac{\rho}2$ where $K \defeq \lceil(\operatorname{diam}(\Theta) / \rho)^d\rceil$. 
Denote the $VC$ dimension of $\Theta$ and $\Theta_k$ as $v$ and $v_k$, respectively.
Then, for any $\vtheta \in \Theta$, the following bound holds with probability at least $1-\delta$:
\begin{multline}\label{eq:swad}
	\mcE(\vtheta; \domainT) \leq \hat{\mathcal{E}}_{\text{SAM}}(\vtheta; \mathcal{D}) + \frac{1}{2} \Div(\domainD, \domainT) \\
	 +  \max_k \consta.
\end{multline}
\end{theorem}
In \cref{eq:swad}, $\Div(\domainD, \domainT)$ is the divergence between source and target domain. 
Building on their results, we show the expected generalisation gap on the target domain over test tasks in FSC can be upper bounded by the gap between the SAM and the ERM loss.

For this, let us assume a global labeling function and let $f_{\vtheta}(\cdot)$ be decomposed as follows: {\em (i)}~task-agnostic functions parametrised by $\vphi$, {\em (ii)}~task-specific functions parametrised by $\vpsi$, and {\em (iii)}~gating functions that switch between task-specific functions and the identity function. 
Note that gating functions are only introduced to make the model contains task-agnostic and task-specific parameters during both training and testing so we can ensure theoretically correctness in the FSC setting.
It does not add any extra operation in practice.
During training on the source domain, all gates are set to choose the identity functions, while during adaptation on the target, all gates are set to select the task-specific layers.
\cref{fig:nnfunction} illustrates our construction of the model functions in the FSC setting with the decomposition of $f_{\vtheta}(\cdot)$ into task-agnostic and task-specific parts.

\begin{figure*}[!t]
	\centering
	
	\small
	\begin{tikzpicture}[outer sep=0pt, every label/.style={align=center}]

        \tikzstyle{blob}=[circle, draw=black!30, black!80, inner sep=1.1pt, fill=white, line width=1pt,font=\small\bf]

		\draw[fill=asgray!20,rounded corners=1mm, draw=none] (0,-2) rectangle (8,2.5);
		\node[below] at (4, -2) {\textbf{Training}};

		\draw[fill=asgray!20,rounded corners=1mm,draw=none] (9,-2) rectangle (17,2.5);
		\node[below] at (13, -2) {\textbf{Adaptation}};

		\node[blob] (step1) at (0.5, 1.5) {1.};

		\node (imagenet) at (1.5, 1.5) {\LARGE\faDatabase};
		\node[label=above:{Backbone}] (imagenet-bone) at (3.5, 1.5) {\LARGE\bbone};
		\draw[-latex] (imagenet) -- node[midway, below] {SAM} (imagenet-bone);

		\node[blob] (step2) at (1.5, -0.75) {2.};

		\node (data1) at (2.5, 0.0) {\textcolor{asblue}{\large\faDatabase$_1$}};
		\node (datai) at (2.4, -0.75) {$\bm{\vdots}$};
		\node (datak) at (2.5, -1.5) {\textcolor{aspurple}{\large\faDatabase$_j$}};

		\node[] (data1-bone) at (3.5, 0.0) {\textcolor{asgray}{\large\bbone}};
		\node (datai-bone) at (3.5, -0.75) {\textcolor{asgray}{$\bm{\vdots}$}};
		\node[] (datak-bone) at (3.5, -1.5) {\textcolor{asgray}{\large\bbone}};

		\node[label=above:{Backbone bank}] (imagenet-bone2) at (6.5, 0.75) {\large\bbone};
		\node[] (data1-bone2) at (6.5, 0.0) {\textcolor{asblue}{\large\bbone$_1$}};
		\node (datai-bone2) at (6.4, -0.75) {$\bm{\vdots}$};
		\node[] (datak-bone2) at (6.5, -1.5) {\textcolor{aspurple}{\large\bbone$_j$}};

		\draw[-latex, dashed] (imagenet-bone) -- node[midway, xshift=15pt] {copy} (data1-bone);

		\draw[-latex] (data1-bone) --  (data1-bone2);
		\draw[draw=none] (datai-bone) -- node[midway, align=center] {Fine-tune\\ w. SAM} (datai-bone2);
		\draw[-latex] (datak-bone) -- (datak-bone2);

		\draw[fill=white, draw=none] (9,2)  {[rounded corners=1mm] -- ++(1,0) } -- ++(0,0.5) {[rounded corners=1mm] -- ++(-1,0)  -- ++(0,-0.5)} -- cycle;
		\node at (9.5, 2.25) {Task $t$};

		\node[blob] (step3) at (9.5, -0.25) {3.};

		\node[label=above:{Estimate\\compatability}] (div-imagenet) at (10.5, 0.75) {{\large\faDatabase} \, $\| \; \mcS_t$ };
		\node (div-data1) at (10.5, 0.0) {\textcolor{asblue}{\large\faDatabase$_1$} $\| \; \mcS_t$ };
		\node (div-dataj) at (10.5, -0.75) {$\bm{\vdots}$};
		\node (div-datak) at (10.5, -1.5) {\textcolor{aspurple}{\large\faDatabase$_j$} $\| \; \mcS_t$ };

		\node[label=right:{Support set}] (support) at (13, 1.5) {\large$\mcS_t$};
		\draw[-latex, dashed,shorten >=4pt] (support) -- (div-imagenet.east) ;

		\node (argmin) at (13, -0.25) {Select best};
		\node (selected-bone) at (14.1, -0.25) {\textcolor{asblue}{\large\bbone}};
		\draw[] (div-imagenet) to[in=180,out=0] (argmin.west);
		\draw[] (div-data1) to[in=180,out=0] (argmin.west);
		\draw[] (div-datak) to[in=180,out=0] (argmin.west);

		\node (adapted-bone) at (16, -0.25) {\textcolor{asturquoise}{\large\bbone}};
		\draw[-latex] (selected-bone) -- node[below, midway, name=adapt-node] {Adapt} (adapted-bone);

		\draw[-latex, dashed] (support) -- ([yshift=2pt]adapt-node.north);

	\end{tikzpicture}\\
	\newcommand{\blob}[1]{\protect\tikz[baseline=(char.base)]\protect\node[circle, draw=black!30, black!80, inner sep=1.1pt, fill=white, line width=1pt,font=\small\bf\tiny](char){#1};}%
	\caption{Our training protocol: \blob{1.}~SAM-based backbone \protect\bbone\ training on a large and diverse data set (\eg, ImageNet), \blob{2.}~SAM-based fine-tuning of \bbone\xspace on additional training data sets, \blob{3.}~backbone selection and adaptation on the selected backbone \textcolor{asblue}{\protect\bbone}$\rightarrow$\textcolor{asturquoise}{\protect\bbone}. \label{fig:protocol} }
\end{figure*}

\begin{theorem} \label{thm:main-theorem}
Let $\vtheta^*_{\text{SAM}}$ denote the optimal solution of the SAM loss $\Ehatsrcsam$, \ie, $\vtheta^*_{\text{SAM}} \defeq \argmin_{\vtheta}\Ehatsrcsam$.
Then, the gap between the loss $\min _{\vtheta} \Etgt$ and the loss of the optimal SAM solution on the training set, $\mcE(\vtheta^*_{\text{SAM}};\domainT)$, has the following bound with probability at least $1-\delta$:
\begin{align}
	\mcE(&\vtheta^*_{\text{SAM}};\domainT)-  \min_{\vtheta} \mcE(\vtheta; \domainT) \nonumber \\
	&\leq \textcolor{asblue}{\hat{\mcE}_{\text{SAM}}(\vtheta^*_{\text{SAM}}; \mcD) - \min_{\vtheta} \hat{\mcE}_{\text{ERM}}(\vtheta; \mcD)} \nonumber\\
	& \quad + \textcolor{asgreen}{\expect_{\domainT_{t} \sim \domainT}\left[\Div(\domainD, \domainT_{t})\right]}  + \textcolor{asgray!60!black}{\constb} \nonumber \\
	& \quad + \textcolor{asgray!60!black}{\max_k \consta} ,\label{eq:bound}
\end{align}
where $N$ is the number of training examples and $\Div(\domainD, \domainT_{t})$ is the divergence between domain $\domainD$ and $\domainT_{t}$. For further details and the proof, see \cref{app:proofs}.
\end{theorem}

Note that by our construction, $\vtheta^*_{\text{SAM}}$ corresponds to the optimal solution under the SAM loss with gates switching to identity functions.
Hence, \cref{eq:bound} holds for any choice of $\vpsi$ and its closeness will depend on the influence of $\vpsi$ on the function value of $f_{\vtheta}$.
In practice, only a very small portion of $f_{\vtheta}$ is task-specific, \eg, for TSA \cite{li2022tsa} less than $1\%$ of the parameters are task-specific, and we can consider the effect of $\vpsi$ to be negligible.

\cref{thm:main-theorem} shows that the expected generalisation gap on the target domain can be bounded by: {\em(i)}~the gap between the SAM and the ERM loss (in \textcolor{asblue}{blue}), {\em(ii)}~the expected discrepancy between the source domain and the target domain (in \textcolor{asgreen}{green}), {\em(iii)}~and confidence bounds depending on $\rho$ (in \textcolor{asgray!60!black}{gray}).
Consequently, in the multi-domain FSC setting, \ie, when we consider a collection of source domains $\{\domainD_1, \dots, \domainD_D\}$, the bound in \cref{thm:main-theorem} suggests that the generalisation gap on the target depends on the selected source domain.
Henceforth, we will suggest a backbone selection mechanism in \cref{sec:method:backbone} to minimise the expected generalisation gap on the target domain.

When the domain discrepancy between $\domainD$ and $\domainT$ is not large, the gap between SAM training loss and empirical training loss will play an important role at the bound. 
Given the complexity of loss landscapes in deep neural networks, it is sensible to assume that SAM with a proper $\rho$ will find an optimal solution with similar training loss as ERM. %

\subsection{Backbone Training}\label{sec:method:backbone}

In this section, we introduce our proposed backbone training protocol for FSC.
First, we use a flatness-seeking objective based on the SAM loss to train the backbone. 
Then, we propose to use a %
fine-tuning strategy to fuse information in the multi-domain settings. 
Finally, we introduce a backbone selection method that we use on unseen domains to choose the best backbone for adaptation.
The training protocol and the respective steps are outlined in \cref{fig:protocol}.

\paragraph{\circled{\small 1.} Flatness Aware Training Objective} 
Motivated by our theoretical result in \cref{sec:method:theory}, we propose to train the backbone with a flatness-seeking objective.
More specifically, we propose to use the SAM objective \cite{foret2020sam} or variants thereof, \eg, Bayesian-SAM (b-SAM) \cite{mollenhoff2022bsam} or adaptive SAM \cite{kwon2021asam}.
This requires minor modification where we optimise the SAM objective rather than ERM when training the backbones, where any gradient-based optimiser can be applied to the SAM objective.  

\paragraph{\circled{\small 2.} Information Fusing using Fine-tuning} 
One key challenge in multi-domain FSC is effective information fusion from different data sets.
Simply training a single backbone on all domains will suffer from task conflicts \cite{li2021url}.
To avoid this, we take inspiration from transfer learning where fine-tuning is %
a simple yet effective way to transfer knowledge between different data sets \cite{pan2009survey,yosinski2014transferable}. 
Specifically, we propose to first train a base backbone on a diverse and extensive training data set (\eg, ImageNet), then fine-tune the trained base backbone on smaller data sets.
We experiment with standard %
fine-tuning and Low-Rank Adaptation (LoRA) \cite{hu2021lora}.  

\begin{table*}[!ht]
	\centering
	\setlength{\tabcolsep}{7.3pt}
	\scriptsize
	\caption{\textbf{Does flatness help generalisation? Yes.} 
		Our performance comparison on the Meta-Dataset indicates that the SAM objective (SAM, b-SAM) results in better generalisation compared to ERM. 
		Each trained backbone is combined with SUR or TSA for adaptation. Performance differences against ERM are indicated by $\textcolor{asgreen}{{\uparrow}}$ and $\textcolor{asred}{{\downarrow}}$. 
		For visual comparison, we include example images from each data set. 
	}
	\begin{tabular}{ll|c|ccc|ccc}
\toprule
&{} & {} & \multicolumn{3}{ c |}{\textit{adapted with SUR}} & \multicolumn{3}{ c }{\textit{adapted with TSA}} \\[0.2em]
&{} & Example Images     &     ERM       &     SAM       &    b-SAM      &     ERM       &     SAM     &    b-SAM       \\\hline
\multirow{9}{*}{\rotatebox{90}{Seen during training}}
& {\sc ilsvrc\_2012}  &   \dataset{imagenet}  & $   55.09{\pm}1.09   $ & $   \bf{56.72{\pm}1.11}  $ & $   56.25{\pm}1.08   $ & $   56.74{\pm}1.08   $ & $   \bf{58.99{\pm}1.08}   $ & $   58.50{\pm}1.06   $ \\
&   {\sc omniglot}    &      \dataset{omniglot}      & $   94.38{\pm}0.44   $ & $   \bf{94.69{\pm}0.44}   $ & $   93.95{\pm}0.46   $ & $   94.64{\pm}0.43   $ & $   \bf{94.87{\pm}0.42}   $ & $   94.19{\pm}0.46   $ \\
&   {\sc aircraft}    &      \dataset{aircraft}     & $   87.68{\pm}0.47   $ & $   \bf{89.51{\pm}0.43}   $ & $   87.74{\pm}0.45   $ & $   88.24{\pm}0.45   $ & $   \bf{89.91{\pm}0.42}   $ & $   88.68{\pm}0.43   $ \\
&  {\sc cu\_birds}    &     \dataset{birds}      & $   72.28{\pm}0.92   $ & $   \bf{74.18{\pm}0.83}   $ & $   72.68{\pm}0.82   $ & $   70.57{\pm}0.86   $ & $   \bf{74.27{\pm}0.81}   $ & $   73.19{\pm}0.83   $ \\
&     {\sc dtd}       &      \dataset{texture}       & $   72.08{\pm}0.73   $ & $   72.99{\pm}0.80   $ & $   \bf{73.08{\pm}0.73}   $ & $   61.33{\pm}0.71   $ & $   \bf{63.43{\pm}0.74}   $ & $   62.73{\pm}0.77   $ \\
&  {\sc quickdraw}    &   \dataset{quickdraw}    & $   83.36{\pm}0.56   $ & $   \bf{83.86{\pm}0.55}   $ & $   83.75{\pm}0.56   $ & $   83.64{\pm}0.58   $ & $   \bf{84.10{\pm}0.57}   $ & $   84.00{\pm}0.58   $ \\
&    {\sc fungi}      &    \dataset{fungi}     & $   68.68{\pm}0.97   $ & $   70.52{\pm}0.95   $ & $   \bf{70.67{\pm}0.95}   $ & $   68.45{\pm}0.96   $ & $   70.50{\pm}0.93   $ & $   \bf{70.54{\pm}0.92}   $ \\
& {\sc vgg\_flower}   &     \dataset{flower}      & $   87.11{\pm}0.52   $ & $   \bf{87.23{\pm}0.51}   $ & $   86.42{\pm}0.54   $ & $   84.28{\pm}0.58   $ & $   \bf{85.30{\pm}0.53}   $ & $   83.49{\pm}0.62   $ \\
\midrule
\multirow{6}{*}{\rotatebox{90}{Unseen}}
&{\sc traffic\_sign}  &      \dataset{traffic}    & $   45.33{\pm}1.02   $ & $   \bf{46.04{\pm}1.05}   $ & $   44.05{\pm}1.15   $ & $   80.73{\pm}0.97   $ & $   \bf{86.07{\pm}0.89}   $ & $   80.87{\pm}0.92   $ \\
&    {\sc mscoco}     &     \dataset{mscoco}     & $   \bf{50.41{\pm}1.03}   $ & $   50.28{\pm}1.08   $ & $   48.54{\pm}1.03   $ & $   56.07{\pm}1.05   $ & $   \bf{57.14{\pm}1.07}   $ & $   55.87{\pm}1.02   $ \\
&    {\sc mnist}      &     \dataset{mnist}      & $   94.71{\pm}0.42   $ & $   94.16{\pm}0.40   $ & $   \bf{95.52{\pm}0.30}   $ & $   96.59{\pm}0.35   $ & $   96.84{\pm}0.34   $ & $   \bf{97.26{\pm}0.32}   $ \\
&   {\sc cifar10}     &     \dataset{cifar10}      & $   68.76{\pm}0.76   $ & $   66.93{\pm}0.86   $ & $   \bf{68.93{\pm}0.94}   $ & $   79.62{\pm}0.73   $ & $   \bf{80.47{\pm}0.71}   $ & $   80.04{\pm}0.71   $ \\
&   {\sc cifar100}    &     \dataset{cifar100}     & $   59.32{\pm}1.05   $ & $   59.31{\pm}1.07   $ & $   \bf{61.21{\pm}1.09}   $ & $   71.21{\pm}0.96   $ & $   \bf{72.28{\pm}0.94}   $ & $   71.95{\pm}0.93   $ \\
\midrule
&    Average Seen     && $77.58$ & $\bf{78.71} \textcolor{asgreen}{{\addspace\uparrow\rdcspace1.13}}$ & $78.07 \textcolor{asgreen}{{\addspace\uparrow\rdcspace0.49}}$ & $75.99$ & $\bf{77.67}\textcolor{asgreen}{{\addspace\uparrow\rdcspace1.68}}$ & $76.91\textcolor{asgreen}{{\addspace\uparrow\rdcspace0.92}}$ \\
&    Average Unseen    && $\bf{63.71}$ & $63.34\textcolor{asred}{{\addspace\downarrow\rdcspace0.37}}$ & $63.65\textcolor{asred}{{\addspace\downarrow\rdcspace0.06}}$ & $76.84$ & $\bf{78.56}\textcolor{asgreen}{{\addspace\uparrow\rdcspace1.72}}$ & $77.20\textcolor{asgreen}{{\addspace\uparrow\rdcspace0.36}}$ \\
\bottomrule\end{tabular}

	\label{table:flatness}
\end{table*}

\paragraph{\circled{\small 3.} Backbone Selection}
We propose using model selection scores on the backbone bank to determine which backbone is most suitable for feature extraction on unseen domain data. 
As indicated in \cref{thm:main-theorem}, the bound suggests that the generalisation gap on the target domain depends on the selected source domain. 
To narrow this gap in the cross-domain FSC setting, we use Pairwise Annotation Representation Comparison (PARC) \cite{bolya2021parc} to select which backbone to use when adapting to unseen domains.  
During evaluation, we calculate the PARC scores for each backbone in the backbone bank on support set, and select the backbone with the highest score for current task.
This only requires forward pass of trained backbones without the need for additional training of the backbones to compute the scores. 
See \cite{bolya2021parc} for more details on PARC.

Our training protocol is adaptation- and backbone-agnostic, methodologically simple to allow easy integration into existing works, and theoretically motivated by \cref{thm:main-theorem}.

\section{Experiments}\label{sec:experiments}
In this section, we first introduce the experimental setup and then study the following questions: {\em(i)}~Does flatness help generalisation in FSC? {\em(ii)}~How does fine-tuning compare against information fusion approaches? {\em(iii)}~How does our proposed training protocol compare with the state-of-the-art methods?

\subsection{Experimental Setup}
We use the Meta-Dataset \cite{metadataset}, a multi-domain FSC benchmark for in- and cross-domain generalisation, including %
data sets introduced by \cite{requeima2019cnaps} for all evaluations.
For this, we follow %
the standard varying-way varying-shot protocol. 
The Meta-Dataset contains the following training data sets: ImageNet \cite{deng2009imagenet}, Omniglot \cite{lake2011omniglot}, Aircraft \cite{maji2013aircraft}, CU\_Birds \cite{wah2011birds}, VGG Flower \cite{nilsback2008flower}, Quickdraw \cite{jongejan2016quickdraw}, Fungi \cite{schroeder2018fungi}, and Describable Textures \cite{cimpoi2014dtd}.
For the cross-domain setting, we evaluated based on Traffic Signs \cite{houben2013trafficsign}, MSCOCO \cite{lin2014mscoco}, MNIST \cite{deng2012mnist}, CIFAR-10 and CIFAR-100 \cite{krizhevsky2009cifar10}. 
To have a fair comparison with prior work, we use a ResNet-18 \cite{he2016deep} as the backbone in main text experiments. 
In \cref{app:results}, we also include results using a Vision Transformer \cite{dosovitskiy2020vit} in \cref{table:vit} to verify that our proposed training procedure performs well for different network architectures. %
For the ResNet-18 backbone, we adopt the recent backbone-agnostic adaptation methods SUR \cite{dvornik2020sur} and TSA \cite{li2022tsa}.
SUR combines features extracted from independently trained backbones and learns combination coefficients on the support set during adaptation. 
TSA adds task-specific feature adapters to the trained backbone and learns these on the support set.  
To ensure a fair comparison, we use the default hyperparameters provided for the adaptation.
For unseen domains we select backbone based on PARC score for each task.
To classify unseen classes during adaptation, we use a nearest-centroid classifier as typically adopted in FSC \cite{snell2017prototypical}.
We compare our backbone training strategy against the following methods: 
{\em(i)}~SUR \cite{dvornik2020sur} and URT \cite{liu2020urt} which learn to fuse information of independent backbones to obtain generalisable features; 
{\em(ii)}~Simple CNAPS (S-CNAPS) \cite{bateni2020simplecnaps} and Transductive CNAPS (T-CNAPS) \cite{bateni2022transductive} which learn an auxiliary network that predicts task-specific parameters of the backbone; 
{\em(iii)}~URL \cite{li2021url} and TSA \cite{li2022tsa} which directly learn the parameters of task-specific layers during adaptation. 
We use a paired $t$-test ($p=0.05$) to bold results with significant statistical difference.
For more details, see \cref{app:implementation}.

\begin{table*}[!h]
	\centering
	\setlength{\tabcolsep}{11pt}
	\scriptsize
	\caption{\textbf{Is fine-tuning effective? Yes.} 
		Our performance comparison on the Meta-Dataset indicates that fine-tuning (LoRA, Vanilla) is an effective fusion strategy as it performs competitively against the other information fusion methods.
		We use TSA as the adaptation method, except for late fusion which is based on SUR. 
		Surprisingly, fine-tuning outperforms knowledge distillation in the cross-domain (unseen) setting. Moreover, we observe that the performance on ImageNet deteriorates after knowledge distillation compared to no fusion.   
	}
	\begin{tabular}{l|c|c:ccaa}
\toprule
{} &($N$, $C$)     &    Late Fusion      &     No Fusion       & Knowledge Distillation &    Ours (LoRA)      &   Ours (Vanilla)     \\\hline
 {\sc ilsvrc\_2012}  &   (11132759, 1000)   & $   55.09{\pm}1.09   $ & $ \bf56.74{\pm}1.08  $ & $   55.67{\pm}1.07   $ & $ \bf56.74{\pm}1.08  $ & $ \bf56.74{\pm}1.08  $ \\
   {\sc omniglot}    &     (32460, 50)      & $   94.38{\pm}0.44   $ & $   94.64{\pm}0.43   $ & $ \bf95.03{\pm}0.41  $ & $   93.48{\pm}0.49   $ & $   94.00{\pm}0.46   $ \\
   {\sc aircraft}    &     (10000, 100)     & $   87.68{\pm}0.47   $ & $   88.24{\pm}0.45   $ & $ \bf89.95{\pm}0.45  $ & $   89.29{\pm}0.48   $ & $ \bf89.94{\pm}0.43  $ \\
  {\sc cu\_birds}    &     (11788, 200)     & $   72.28{\pm}0.92   $ & $   70.57{\pm}0.86   $ & $ \bf82.08{\pm}0.72  $ & $   80.84{\pm}0.74   $ & $   81.46{\pm}0.68   $ \\
     {\sc dtd}       &      (5640, 47)      & $   72.08{\pm}0.73   $ & $   61.33{\pm}0.71   $ & $ \bf75.63{\pm}0.67  $ & $   74.80{\pm}0.76   $ & $   74.35{\pm}0.74   $ \\
  {\sc quickdraw}    &   (50426266, 345)    & $   83.36{\pm}0.56   $ & $ \bf83.64{\pm}0.58  $ & $   82.33{\pm}0.62   $ & $   80.96{\pm}0.65   $ & $   83.08{\pm}0.60   $ \\
    {\sc fungi}      &    (89760, 1394)     & $ \bf68.68{\pm}0.97  $ & $   68.45{\pm}0.96   $ & $   67.62{\pm}0.97   $ & $   61.39{\pm}1.04   $ & $   68.23{\pm}0.96   $ \\
 {\sc vgg\_flower}   &     (8189, 102)      & $   87.11{\pm}0.52   $ & $   84.28{\pm}0.58   $ & $ \bf92.90{\pm}0.43  $ & $   92.41{\pm}0.45   $ & $   92.33{\pm}0.41   $ \\
\midrule
{\sc traffic\_sign}  &     (39209, 43)      & $   45.33{\pm}1.02   $ & $   80.73{\pm}0.97   $ & $ \bf81.51{\pm}0.97  $ & $   80.73{\pm}0.97   $ & $   80.73{\pm}0.97   $ \\
    {\sc mscoco}     &     (860001, 80)     & $   50.41{\pm}1.03   $ & $ \bf56.07{\pm}1.05  $ & $   53.98{\pm}1.07   $ & $ \bf56.07{\pm}1.05  $ & $ \bf56.07{\pm}1.05  $ \\
    {\sc mnist}      &     (10000, 10)      & $   94.71{\pm}0.42   $ & $   96.59{\pm}0.35   $ & $   96.65{\pm}0.37   $ & $ \bf97.21{\pm}0.31  $ & $   97.00{\pm}0.34   $ \\
   {\sc cifar10}     &     (10000, 10)      & $   68.76{\pm}0.76   $ & $ \bf79.62{\pm}0.73  $ & $   78.93{\pm}0.77   $ & $ \bf79.62{\pm}0.73  $ & $ \bf79.62{\pm}0.73  $ \\
   {\sc cifar100}    &     (10000, 100)     & $   59.32{\pm}1.05   $ & $ \bf71.21{\pm}0.96  $ & $   70.11{\pm}1.00   $ & $ \bf71.21{\pm}0.96  $ & $ \bf71.21{\pm}0.96  $ \\
\midrule
    Average Seen     && $77.58$ & $75.99$ & $\bf{80.15}$ & $78.74$ & $80.02$ \\
   Average Unseen    && $63.71$ & $76.84$ & $76.24$ & $\bf{76.97}$ & $76.92$ \\
\bottomrule\end{tabular}

	\label{table:fine-tune}
\end{table*}

\begin{table}[!h]
	\centering
	\setlength{\tabcolsep}{5pt}
	\scriptsize
	\caption{Trace and top eigenvalues of the Hessian of the loss to measure the flatness of the trained backbones. Lower values mean flatter solutions. SAM in general results in flatter backbones compared with ERM. %
	}
	\begin{tabular}{l|cc|cc}
	\toprule
	{} & \multicolumn{2}{ c |}{\textit{Trace of Hessian}} & \multicolumn{2}{ c }{\textit{Top Eigenvalues of Hessian}} \\[0.2em]
	{} &ERM    & SAM      &ERM    & SAM    \\ \hline
	{\sc ilsvrc\_2012} & $\bf8873.46$ & $13780.30$ & $\bf237.13$ & $373.42$ \\ 
	{\sc omniglot}     & $199.15$  & $\bf169.05$   & $9.39$   & $\bf8.09$   \\ 
	{\sc aircraft}     & $274.93$  & $\bf170.70$   & $12.71$ & $\bf11.53$  \\
	{\sc cu\_birds}    & $\bf229.89$  & $341.53$   & $\bf4.37$  & $10.60$  \\
	{\sc dtd}          & $134.64$  & $\bf63.58$    & $7.60$   & $\bf2.67$   \\
	{\sc quickdraw}    & $3508.08$ & $\bf1785.32$  & $59.79$  & $\bf37.50$  \\ 
	{\sc fungi}        & $5359.20$ & $\bf3998.80$  & $162.77$ & $\bf149.88$ \\ 
	{\sc vgg\_flower}  & $130.46$  & $\bf89.39$    & $7.03$   & $\bf6.66$   \\ 
	\bottomrule
\end{tabular}

	\label{table:hessian}
\end{table}

\subsection{Does Flatness Help Generalisation in FSC?} 
\label{sec:experiments:flatness}

To evaluate whether the flatness-seeking training leads to more generalisable backbones, we compare backbones trained with the SAM and ERM objectives, respectively. 
For the SAM objective, we use vanilla SAM \cite{foret2020sam} and b-SAM \cite{mollenhoff2022bsam}. 
To assess the performance of varying adaptation methods, we employ both SUR and TSA.  
For SUR, the multi-domain features are fused during adaptation, while TSA needs backbone selection for the unseen domains. 
Hence, for TSA we employ our suggested backbone selection based on PARC~\cite{bolya2021parc}. 
We evaluate whether our backbone selection strategy chose the compatible backbone for unseen domain and report results in \cref{sec:backbone_selection}. 
As shown in \cref{table:parc}, PARC selects the most compatible backbone.

In \cref{table:flatness}, we observe that using the SAM objective during backbone training results in better generalisation on both seen and unseen domains in most cases. 
In particular, both SAM and b-SAM combined with TSA achieve better average performance on the seen and unseen domains compared to backbones trained with ERM. 
Note that for b-SAM, we are only using the posterior mean for computational reasons which might explain why it underperforms against SAM. 
For SUR, the information fusion combined with SAM might cause side effects that result in a slight drop in performance for the unseen domains. 
Further, the improvements on seen domains are larger than on the unseen domains in general. 
This performance gap is potentially caused by large domain shifts which would align with our theoretic findings, {\it c.f.} \cref{thm:main-theorem}. 
Nevertheless, these results indicate that seeking flat minima during backbone training can improve generalisation in FSC. 

Additionally, we measure the flatness of backbones trained with SAM and ERM using the trace and top eigenvalues of the Hessian of the loss \cite{foret2020sam}.
\cite{wen2022samhessian} show that in the full-batch setting, SAM provably decreases the largest eigenvalue of Hessian, while in the stochastic setting (when batch size is 1), SAM provably decreases the trace of Hessian.
 Though in our experiment we use mini-batches where their theoretical results is inapplicable, we report trace and top eigenvalues of Hessian of trained backbones as they still measure the flatness in some degree. 
 As shown in \cref{table:hessian}, SAM finds flatter solution in general compared with ERM.

\subsection{Is Fine-tuning Enough for Information Fusion?} %
To evaluate how well %
fine-tuning performs for information fusion, we compare it against recent fusion methods. 
Our fine-tuning strategy on Meta-Dataset involves two steps: 
{\em (i)}~train one backbone on the ImageNet data set, and
{\em (ii)}~fine-tune copies of the ImageNet-trained backbone on the remaining training data sets. 
We experiment with vanilla fine-tuning referred to as \textbf{Vanilla}, as well as fine-tuning using \textbf{LoRA} \cite{hu2021lora}. 
We compare against the following methods: 
\begin{itemize}[leftmargin=*,itemsep=0pt] 
	\item \textbf{Late Fusion:} The multi-domain feature representation from SUR that fuses information from all backbones.
	\item \textbf{No Fusion:} Using single-domain backbones directly.
	\item \textbf{Knowledge Distillation:} The backbone from URL learned from distilling information from all backbones. 
\end{itemize}
Note that all methods use TSA for adaptation, except for late fusion which is based on SUR. 
Moreover, all backbones are trained with the ERM objective.

\begin{table*}[!t]
	\centering
	\renewcommand*{\arraystretch}{1.1}
	\setlength{\tabcolsep}{5.5pt}
	\scriptsize
	\caption{\textbf{A new baseline for FSC.} Our assessment on the Meta-Dataset shows that SAM-based training combined with fine-tuning (SAM+FT) outperforms SoTA methods in 10 out of 13 domains. Moreover, using SAM in conjunction with other information fusion methods, \eg, knowledge distillation (SAM+KD), can improve generalisation performance. 
		Henceforth, we advocate that our simple yet effective training procedure should be considered as a competitive baseline for both in-domain and cross-domain FSC.
	}
	\begin{tabular}{l|ccccccaa}
\toprule
{} &      S-CNAPS        &      T-CNAPS        &        SUR          &        URT          &        URL          &        TSA          &   Ours (SAM+KD)     &   Ours (SAM+FT)     \\\hline
 {\sc ilsvrc\_2012}  & $   56.03{\pm}1.11   $ & $   56.61{\pm}1.08   $ & $   55.09{\pm}1.09   $ & $   55.17{\pm}1.08   $ & $   55.65{\pm}1.07   $ & $   55.67{\pm}1.07   $ & $   57.03{\pm}1.07   $ & $ \bf59.01{\pm}1.08  $ \\
   {\sc omniglot}    & $   91.45{\pm}0.62   $ & $   92.91{\pm}0.50   $ & $   94.38{\pm}0.44   $ & $   94.42{\pm}0.46   $ & $   94.76{\pm}0.41   $ & $ \bf95.03{\pm}0.41  $ & $ \bf95.03{\pm}0.42  $ & $   94.46{\pm}0.43   $ \\
   {\sc aircraft}    & $   80.90{\pm}0.73   $ & $   82.11{\pm}0.63   $ & $   87.68{\pm}0.47   $ & $   88.16{\pm}0.47   $ & $   89.57{\pm}0.45   $ & $   89.95{\pm}0.45   $ & $   89.34{\pm}0.46   $ & $ \bf92.63{\pm}0.35  $ \\
  {\sc cu\_birds}    & $   75.10{\pm}0.86   $ & $   77.35{\pm}0.77   $ & $   72.28{\pm}0.92   $ & $   79.04{\pm}0.77   $ & $   81.51{\pm}0.69   $ & $   82.08{\pm}0.72   $ & $   82.58{\pm}0.70   $ & $ \bf85.57{\pm}0.60  $ \\
     {\sc dtd}       & $   68.90{\pm}0.72   $ & $   68.76{\pm}0.73   $ & $   72.08{\pm}0.73   $ & $   73.05{\pm}0.67   $ & $   74.66{\pm}0.65   $ & $   75.63{\pm}0.67   $ & $ \bf76.96{\pm}0.66  $ & $   75.35{\pm}0.72   $ \\
  {\sc quickdraw}    & $   77.53{\pm}0.77   $ & $   78.74{\pm}0.67   $ & $   83.36{\pm}0.56   $ & $ \bf83.51{\pm}0.56  $ & $   82.42{\pm}0.61   $ & $   82.33{\pm}0.62   $ & $   82.79{\pm}0.61   $ & $   83.30{\pm}0.59   $ \\
    {\sc fungi}      & $   48.07{\pm}1.10   $ & $   48.39{\pm}1.12   $ & $   68.68{\pm}0.97   $ & $   68.23{\pm}0.99   $ & $   68.44{\pm}0.98   $ & $   67.62{\pm}0.97   $ & $   67.95{\pm}0.96   $ & $ \bf70.13{\pm}0.91  $ \\
 {\sc vgg\_flower}   & $   91.45{\pm}0.52   $ & $   92.25{\pm}0.45   $ & $   87.11{\pm}0.52   $ & $   90.09{\pm}0.46   $ & $   91.55{\pm}0.44   $ & $   92.90{\pm}0.43   $ & $ \bf93.41{\pm}0.43  $ & $ \bf93.59{\pm}0.39  $ \\
\midrule
{\sc traffic\_sign}  & $   58.33{\pm}1.03   $ & $   56.83{\pm}1.13   $ & $   45.33{\pm}1.02   $ & $   47.11{\pm}1.02   $ & $   60.15{\pm}1.18   $ & $   81.51{\pm}0.97   $ & $   82.83{\pm}0.93   $ & $ \bf86.04{\pm}0.89  $ \\
    {\sc mscoco}     & $   48.79{\pm}1.09   $ & $   50.89{\pm}1.05   $ & $   50.41{\pm}1.03   $ & $   50.15{\pm}1.03   $ & $   52.82{\pm}1.01   $ & $   53.98{\pm}1.07   $ & $   55.20{\pm}1.07   $ & $ \bf57.13{\pm}1.07  $ \\
    {\sc mnist}      & $   93.81{\pm}0.42   $ & $   95.13{\pm}0.30   $ & $   94.71{\pm}0.42   $ & $   88.89{\pm}0.48   $ & $   94.84{\pm}0.42   $ & $   96.65{\pm}0.37   $ & $   96.71{\pm}0.37   $ & $ \bf97.05{\pm}0.30  $ \\
   {\sc cifar10}     & $   71.75{\pm}0.76   $ & $   72.60{\pm}0.68   $ & $   68.76{\pm}0.76   $ & $   64.50{\pm}0.75   $ & $   69.93{\pm}0.74   $ & $   78.93{\pm}0.77   $ & $   80.07{\pm}0.75   $ & $ \bf80.60{\pm}0.71  $ \\
   {\sc cifar100}    & $   61.62{\pm}1.05   $ & $   62.35{\pm}1.02   $ & $   59.32{\pm}1.05   $ & $   55.85{\pm}1.07   $ & $   61.58{\pm}1.09   $ & $   70.11{\pm}1.00   $ & $   71.34{\pm}0.98   $ & $ \bf72.38{\pm}0.95  $ \\
\midrule
    Average Seen     & $73.68$ & $74.64$ & $77.58$ & $78.96$ & $79.82$ & $80.15$ & $80.64$ & $\bf{81.75}$ \\
   Average Unseen    & $66.86$ & $67.56$ & $63.71$ & $61.30$ & $67.86$ & $76.24$ & $77.23$ & $\bf{78.64}$ \\
\bottomrule\end{tabular}

	\label{table:compare-with-sota}
\end{table*}

In \cref{table:fine-tune}, we observe that our %
fine-tuning strategy performs competitively against other information fusion methods, especially on the unseen domains in Meta-Dataset. 
Compared to no fusion, both LoRA and vanilla fine-tuning yield better backbones on the seen domains, which means that the fine-tuning fuses information from ImageNet with the different domains successfully. 
On the unseen domains, our method and no fusion both select the ImageNet-backbone for the color datasets and Omniglot-backbone for MNIST, which is why their accuracies are similar. \looseness-1

When compared to the information fusion methods, our vanilla fine-tuning outperforms late fusion and performs competitively against knowledge distillation on both seen and unseen domains. 
While the smaller data sets benefit from using the universal backbone from knowledge distillation, we observe that our vanilla performs slightly better on the larger domains QuickDraw and Fungi.  
Furthermore, our method performs better than Knowledge Distillation on all unseen domains except Traffic Signs, which could be because the fine-tuning strategy mitigates the risk of task conflict when fusing the seen domains with ImageNet.  
These results demonstrate that our fine-tuning strategy is an effective alternative to previous information fusion methods.

\subsection{How Does Our Approach Compare with SoTA?}
We combine SAM-based backbone training with fine-tuning, which we denote as \textbf{SAM+FT}, and assess its performance by comparing it against SoTA on the Meta-Dataset.
In addition, we combine SAM training with knowledge distillation based on URL \cite{li2021url}, denoted as \textbf{SAM+KD}, to evaluate the effect of SAM on SoTA information fusion.
We report the results with TSA adaptation for both approaches in \cref{table:compare-with-sota}.

We observe that our training protocol (SAM+FT) outperforms SoTA methods on most domains (10 out of 13) despite its simplicity.
Furthermore, accounting for flatness in knowledge distillation (SAM+KD) results in mild improvements over TSA but is lacking behind our proposed approach.
Our results show that flatness can improve generalisation in FSC across different information fusion strategies.
To this end, we suggest that our simple yet effective training procedure be considered as a competitive baseline in FSC.

\section{Related Work}\label{sec:relatedwork}

\paragraph{Sharpness-aware minimisation} 
The geometry of minima in neural network training has long been hypothesized to influence the generalisation behaviour of neural networks (\eg, \cite{hochreiter1997flat,keskar2016flatness}).
Consequently, recent works studied theoretical links between flatness and generalisation of neural networks. 
Various algorithms accounting for the loss geometry have been proposed.
For example, \cite{keskar2016flatness} showed a negative correlation between the sharpness of the loss landscape and the generalisation ability of the learner.
Later, \cite{dinh2017sharp} related sharpness to the spectrum of the Hessian and \cite{dziugaite2017flatness} proposed a PAC-Bayes bound-based optimisation scheme to find flat minima.
Recently, \cite{wang2023sharpness} introduced an augmented SAM loss which aims to further encourage flat minima and proposed a respective optimiser.

In the context of domain generalisation, \cite{cha2021swad} showed that flat minima can lead to a smaller generalisation gap on the target domain by leveraging results on generalisation in domain adaptation \cite{ben2010theory}.
In the few-shot learning setting, \cite{shi2021samfewshotapp1} showed that flatness can help in overcoming catastrophic forgetting in the incremental learning setting and \cite{fan2023samfewshotapp2} recently proposed to account for flatness in the prompt tuning of large language models.
However, to the best of our knowledge, sharpness-aware loss functions have not been leveraged or analysed in FSC settings.

\paragraph{Fine-tuning} 
Transfer learning \cite{pan2009survey} involves utilizing and transferring knowledge learned from a set of source tasks to an unseen target task. 
For this, fine-tuning is a commonly adopted strategy where a pre-trained neural network, or a subset of its layers, is adapted to the target task. 
For example, \cite{yosinski2014transferable} showed that fine-tuning a pre-trained network on a new data set can lead to better generalisation compared to training from scratch.
For multi-domain FSC, avoiding task conflicts between different data sets during information fusing is an important problem.
Motivated by the effectiveness of fine-tuning in transfer learning, fine-tuning has been adopted in the FSC setting.

In particular, fine-tuning has been shown to outperform elaborate methods in the adaptation stage (\eg, \cite{chen2019closer,dhillon2020baseline,tian2020rethinking}), in cases where the target domain is similar to the source domain (in-domain setting).
Further, recent works showed that fine-tuning can be a successful adaptation strategy in both in-domain and cross-domain settings \cite{guo2020broader,hu2022pushing,luo2023closeragain} and obtain competitive results compared with elaborate adaptation methods.
However, fine-tuning the backbone before adaptation has received little to no attention. \looseness -1

\paragraph{Existing methods for Meta-Dataset} FSC methods mainly focuses on two perspectives: {\em(i)}~task-agnostic backbone training; {\em(ii)}~adapting the task-agnostic backbone into a task-specific few-shot classifier.

For task-agnostic backbone training, FLUTE \cite{triantafillou2021flute} trains a shared backbone jointly with domain-specific feature adapters on all training domains. 
It entangles adaptation with backbone training, limiting its applicability to other adaptation strategies.
Later, \cite{li2021url} proposed to use knowledge distillation for information fusion, disentangling the backbone training from the adaptation. 
However, knowledge distillation is computationally expensive, and distilling knowledge from multiple domains simultaneously can suffer from task conflict.

For adaptation strategies, \cite{requeima2019cnaps} proposed CNAPS which learns an auxiliary network that predicts parameters of task-specific layers. 
Subsequently, \cite{bateni2020simplecnaps} replaced the linear classifier in CNAPS with a nearest-centroid classifier using Mahalanobis distance inspired by ProtoNetwork \cite{snell2017prototypical}, which was later extended with transductive learning \cite{bateni2022transductive}.
However, learning an effective auxiliary network is difficult and its performance may suffer from distribution shift.  
Different approaches to adaptation are late fusion through linear combinations of features extracted from fixed backbones \cite{dvornik2020sur,liu2020urt}, or learning task-specific layers through optimisation \cite{li2022tsa}.  
Nevertheless, little effort has been devoted to adaptation-agnostic backbone training. \looseness-1

\section{Discussion and Conclusion}
In this work, we have shown that flatness-seeking objectives, such as the SAM loss \citep{foret2020sam}, can improve generalisation in few-shot classification (FSC).
Combined with vanilla fine-tuning, minimising the SAM loss instead of the empirical risk results in a competitive baseline that outperforms current state-of-the-art methods in 10 out of 13 cases on the Meta-Dataset \cite{metadataset} benchmark (see \cref{table:compare-with-sota}).

In particular, we theoretically show that the generalisation gap on the target domain is upper bounded by the gap between the SAM and the ERM loss on the source domain and the difference between the domains.
Motivated by this result, we proposed a %
backbone training protocol consisting of three steps: {\em (i)}~SAM-based backbone training, {\em (ii)}~information fusion using fine-tuning of the backbone(s), {\em (iii)}~backbone selection in the multi-domain setting for unseen domains.  
We empirically showed that our approach is effective despite being methodologically simple, %
and that it can be combined with any adaptation method in FSC. 
Furthermore, we demonstrated that any information fusion method can potentially benefit from flat minima.

\paragraph{Limitations}
Our empirical findings are limited to the Meta-Dataset benchmark in which most data sets contain natural images or black-and-white drawings and written characters.
As illustrated by the examples in \cref{table:flatness}, coloured training data sets can be considered similar (or even sub-sets) of ImageNet, and unseen domains are close to the in-domain data sets.
Our additional results in \cref{table:backbone-selection} confirm this by highlighting the importance of the ImageNet backbone in the backbone selection.
Moreover, our backbone selection strategy requires multiple forward passes, and the limitations of PARC \citep{bolya2021parc} apply to our method.
Lastly, SAM-based optimisation brings additional computational costs during backbone training.

\paragraph{Future directions} 
Based on the improvement we have shown, investigating whether flatness helps generalisation in different model structures, \eg, foundation models, and the wider few-shot learning context is a promising direction.
Further, in consideration of the limitation of Meta-Dataset, it is important to investigate the performance of FSC methods in cross-domain settings. 
In particular, the generalisation gap on target domains with larger distribution shifts compared to the training data sets.
Additional future directions include: leveraging the uncertainty estimates from b-SAM for more robust adaptation and backbone selection, improving the scalability of our approach, and assessing the performance in real-world downstream settings.

\paragraph{Code} Publicly available under MIT license: \url{https://github.com/AaltoML/FlatFSL}

\clearpage
\paragraph{Acknowledgement} 
AS and RL acknowledge funding from the Research Council of Finland (grant number 339730). MT acknowledges funding from the Research Council of Finland (grant number 347279). MK acknowledges funding from the Finnish Center for Artificial Intelligence (FCAI). We acknowledge CSC -- IT Center for Science, Finland, for awarding this project access to the LUMI supercomputer, owned by the EuroHPC Joint Undertaking, hosted by CSC (Finland) and the LUMI consortium through CSC. We acknowledge the computational resources provided by the Aalto Science-IT project.

\bibliographystyle{ieee_fullname}

\newpage
\appendix
\onecolumn
\section*{Appendices}

\section{Proofs}\label{app:proofs}
To prove \cref{thm:main-theorem}, we first prove \cref{lemma3,lemma4}. Then we use \cref{lemma3,lemma4} to prove \cref{lemma1}. After that, we use \cref{lemma1} to prove \cref{lemma2}. At last, we use \cref{lemma2} to prove \cref{thm:main-theorem}.

For simplicity, in the bounded instance function $\loss: \mathcal{Y} \times \mathcal{Y} \rightarrow[0,m]$ where $\loss\left(y_1, y_2\right)=0$ if and only if $y_1=y_2$, we set $m=1$ in the proof. 
We use proof techniques and results from \cite{cha2021swad}. 

\subsection{Proof of \cref{lemma3,lemma4}}
We prove \cref{lemma3,lemma4} in this subsection.
\newcommand{\PD}{\mathbb{P}_{\domainD}}
\newcommand{\PT}{\mathbb{P}_{\domainT^{(i)}}}

\begin{lemma}\label{lemma3}
	Define a functional error for two functions $f_1(\cdot)$ and $f_2(\cdot)$ on a domain $\domainD$ as 
	\begin{equation}
		\mcE(f_1, f_2; \domainD) \defeq \expect_{\vx \sim \PD}[\loss(f_1(\vx), f_2(\vx))].
	\end{equation}
For any domain $\domainD$ and $\domainT^{(i)}$, we have $|\mathcal{E}_{\domainD}(f_1, f_2)-\mathcal{E}_{\domainT}(f_1, f_2)| \leq \frac12\Div(\domainD, \domainT^{(i)}).$
\end{lemma}

\begin{proof}
From the Fubini's theorem, we have
\begin{equation}
\expect_{\vx \sim \PD}\left[\loss \left(f_1(\vx), f_2(\vx)\right)\right]=\int_0^{\infty} \PD\left(\loss\left(f_1(\vx), f_2(\vx)\right)>t\right) \mathrm{d} t.
\end{equation}
Then, 
\begin{equation}
\begin{aligned}
& \left|\expect_{\vx \sim \PD}\left[\loss\left(f_1(\vx), f_2(\vx)\right)\right]-\expect_{\vx^{\prime} \sim \PT}\left[\loss\left(f_1(\vx^{\prime}), f_2(\vx^{\prime})\right)\right]\right| \\
& =\left|\int_0^{\infty} \PD\left(\loss\left(f_1(\vx), f_2(\vx)\right)>t\right)  \mathrm{d} t-\int_0^{\infty} \PT\left(\loss\left(f_1(\vx^{\prime}), f_2(\vx^{\prime})\right)>t\right)  \mathrm{d} t\right| \\
& \leq \int_0^{\infty}\left|\PD\left(\loss\left(f_1(\vx), f_2(\vx)\right)>t\right)-\PT\left(\loss\left(f_1(\vx^{\prime}), f_2(\vx^{\prime})\right)>t\right)\right|  \mathrm{d} t \\
& \leq M \sup _{t \in[0, M]}\left|\PD\left(\loss\left(f_1(\vx), f_2(\vx)\right)>t\right)-\PT\left(\loss\left(f_1(\vx^{\prime}), f_2(\vx^{\prime})\right)>t\right)\right| \\
& \leq M \sup _{f_1, f_2} \sup _{t \in[0, M]}\left|\PD\left(\loss\left(f_1(\vx), f_2(\vx)\right)>t\right)-\PT\left(\loss\left(f_1(\vx^{\prime}), f_2(\vx^{\prime})\right)>t\right)\right| \\
& \leq M \sup_{A \in \mathcal{A}} |\PD(A) - \PT(A)|\\
& \defeq \frac12\Div(\domainD, \domainT^{(i)}),
\end{aligned}
\end{equation}
where $\mathcal{A} = \{\vx, \vx^{\prime} \mid \ell(f_1(\vx), f_2(\vx)) > t, \ell(f_1(\vx^{\prime}), f_2(\vx^{\prime})) > t, \text{  for  } t \in [0, M]\}.$  
\end{proof}

\paragraph{Remark} Because in FSC the source data and target data always have disjoint classes, we could assume there is a global labelling function $h(\cdot)$ for both source and target domain. Then in \cref{lemma3}, if we let $f_1(\vx)$ be the model $f_{\vtheta}(\vx)$ and $f_2(\vx)$ be the global labelling function $h(\cdot)$, $\mathcal{E}(f_1, f_2; \domainD)$ becomes $\Esrc$ and $\mathcal{E}_{\domainT^{(i)}}(f_1, f_2)$ becomes $\Etgti$, and we have 
\begin{equation}\label{eq:lemma3}
	|\mcE(\vtheta; \domainD) - \mcE(\vtheta; \domainT)| \leq \frac12 \Div(\domainD, \domainT).
\end{equation}

\begin{lemma}\label{lemma4}
Let $\vtheta_k \in \arg \max _{\Theta_k \cap \Theta} \Esrc$ be a local maximum in the $k$-th ball $\Theta_k$. For any $\vtheta \in \Theta$, the following bound holds with probability at least $1-\delta$:
\begin{equation}
\Esrc-\Ehatsrcsam \leq \max_k \consta
\end{equation}
where $N$ is the number of data points in the training set.
\end{lemma}

\begin{proof}
We first prove the following inequality holds
\begin{equation}
	\Esrc-\Ehatsrcsam \leq \max_k \left[\mathcal{E}\left(\vtheta_{k^{\prime}};\domainD\right) - \hat{\mathcal{E}}\left(\vtheta_{k^{\prime}};\mcD\right) \right],
\end{equation}
then we prove for ${\varepsilon_k \defeq \consta}, \varepsilon \defeq \max_k \varepsilon_k$, we have $\mathbb{P}\left(\max _k\left[\mathcal{E}\left(\vtheta_k; \domainD\right)-\hat{\mathcal{E}}\left(\vtheta_k;\mcD\right)\right]>\varepsilon\right) \leq \delta$.

Since for any $\vtheta$ there exists ${k^{\prime}}$ such that $\vtheta \in \Theta_{k^{\prime}}$, we have
\begin{equation}
\begin{aligned}
\Esrc-\Ehatsrcsam &= \Esrc - \max _{\|\mbd{\epsilon}\| \leq \rho} \hat{\mathcal{E}}_{\mathcal{D}}(\vtheta+\mbd{\epsilon}) \\
& \leq \mathcal{E}(\vtheta;\domainD)-\hat{\mathcal{E}}\left(\vtheta_{k^{\prime}}; \mcD\right) \\
& =  \mathcal{E}(\vtheta;\domainD) - \mathcal{E}\left(\vtheta_{k^{\prime}};\domainD\right) + \mathcal{E}\left(\vtheta_{k^{\prime}};\domainD\right) - \hat{\mathcal{E}}\left(\vtheta_{k^{\prime}}; \mcD\right) \\
& \leq \mathcal{E}\left(\vtheta_{k^{\prime}};\domainD\right) - \hat{\mathcal{E}}\left(\vtheta_{k^{\prime}};\mcD\right) \\
& \leq \max_k \left[\mathcal{E}\left(\vtheta_{k^{\prime}};\domainD\right) - \hat{\mathcal{E}}\left(\vtheta_{k^{\prime}};\mcD\right) \right],
\end{aligned}
\end{equation}
where the second inequality holds because $\vtheta_k$ is the local maximum in $\Theta_k$. 

We now prove $\mathbb{P}\left(\max _k\left[\mathcal{E}\left(\vtheta_k;  \domainD\right)-\hat{\mathcal{E}}\left(\vtheta_k;\mcD\right)\right]>\varepsilon\right) \leq \delta$. To do so, we first show the following inequality holds for the local maximum of $N$ covers:
\begin{equation}
\begin{aligned}
\mathbb{P}\left(\max_k\left[\mathcal{E}\left(\vtheta_k;\domainD\right)-\hat{\mathcal{E}}\left(\vtheta_k; \mcD\right)\right]>\varepsilon\right) & \leq \sum_{k=1}^K \mathbb{P}\left(\mathcal{E}(\vtheta_k;\domainD)-\hat{\mathcal{E}}(\vtheta_k;\mcD) >\varepsilon \right) \\
& \leq \sum_{k=1}^K \mathbb{P}\left(\sup _{\vtheta \in \Theta_k}\left[\mathcal{E}(\vtheta;\domainD)-\hat{\mathcal{E}}(\vtheta;\mcD)\right]>\varepsilon\right) \\
& \leq \sum_{k=1}^K\left(\frac{eN}{v_k}\right)^{v_k} e^{-2 N \epsilon^2}.
\end{aligned}
\end{equation}
Then by the definition of $\varepsilon$, we have
\begin{equation}
\begin{aligned}
\mathbb{P}\left(\max_k\left[\mathcal{E}\left(\vtheta_k;\domainD\right)-\hat{\mathcal{E}}\left(\vtheta_k;\mcD\right)\right]>\varepsilon\right) & \leq \sum_{k=1}^K\left(\frac{eN}{v_k}\right)^{v_k} e^{-2N \varepsilon^2} \\
& \leq \sum_{k=1}^K\left(\frac{eN}{v_k}\right)^{v_k} e^{-2 N \varepsilon_k^2} \\
& =\sum_{k=1}^K \frac{\delta}{N}=\delta.
\end{aligned}
\end{equation}

Since $\Esrc-\Ehatsrcsam \leq \max_k \left[\mathcal{E}\left(\vtheta_k;\domainD\right) - \hat{\mathcal{E}}\left(\vtheta_k;\mcD\right) \right]$ and $\mathbb{P}\left(\max_k\left[\mathcal{E}\left(\vtheta_k;\domainD\right)-\hat{\mathcal{E}}\left(\vtheta_k; \mcD\right)\right]>\varepsilon\right)  \leq \delta$, the following inequality holds with probability at least $1-\delta$:

\begin{equation}
\begin{aligned}
	\Esrc-\Ehatsrcsam & \leq \max_k \left[\mathcal{E}\left(\vtheta_k; \domainD\right) - \hat{\mathcal{E}}\left(\vtheta_k; \mcD\right) \right] \\
	& \leq	\varepsilon = \max_k \consta.
\end{aligned}
\end{equation}

\end{proof}

\subsection{Proof of \cref{lemma1}}
We prove \cref{lemma1} using \cref{lemma3,lemma4} in this subsection.

\begin{lemma}\label{lemma1}
	Let $\left\{\Theta_k \subset \mathbb{R}^d, k=1, \cdots, K\right\}$ ($d$ is dimension of $\Theta$) be a finite cover of a parameter space $\Theta$ which consists of $K$ closed balls with radius $\nicefrac{\rho}2$ where $K \defeq \lceil(\operatorname{diam}(\Theta) / \rho)^d\rceil$. 
	Let $v_k$ be a $VC$ dimension of each $\Theta_k$. 
	Then, for any $\vtheta \in \Theta$, the following bound holds with probability at least $1-\delta$,
	\begin{equation}\label{eq:theorem1}
	\begin{aligned}
		\mathcal{E}(\vtheta; \domainT) \leq & \Ehatsrcsam + \frac12\Div(\domainD, \domainT^{(i)})  \\
		&+\max_k \consta.
	\end{aligned}
	\end{equation}
	
\end{lemma}

\begin{proof}
We consider two cases.

\paragraph{Case 1:} $\Etgti \leq \Esrc$

As $\Etgti - \Esrc \leq 0$, \cref{lemma1} automatically holds.

\paragraph{Case 2:} $\Etgti > \Esrc$

From \cref{lemma3,eq:lemma3}, we have
\begin{equation}\label{eq:domaindt}
	| \Esrc - \Etgti| = \Etgti - \Esrc \leq  \frac12\Div(\domainD, \domainT^{(i)}).
\end{equation}

Combining it with \cref{lemma4}, we have 
\begin{equation}
\begin{aligned}
	\Etgti   &\leq  \Esrc + \frac12\Div(\domainD, \domainT) \\
	 \Etgti  & \leq  \Esrc - \Ehatsrcsam  + \Ehatsrcsam  + \frac12\Div(\domainD, \domainT) \\
	\Etgti & \leq \Ehatsrcsam + \max_k \consta + \frac12\Div(\domainD, \domainT). \\
\end{aligned}
\end{equation}
\end{proof}

\subsection{Proof of \cref{lemma2}}
We prove \cref{lemma2} using \cref{lemma1} in this subsection.

\begin{lemma} \label{lemma2}
	Denote the $VC$ dimension of $\Theta$ as $v$. 
	Let $\vtheta^*_{\text{SAM}}$ denote the optimal solution of the SAM loss $\Ehatsrcsam$, i.e., $\vtheta^*_{\text{SAM}} \defeq \argmin_{\vtheta}\Ehatsrcsam$. 
	Then, the gap between the optimal test loss, $\min_{\vtheta} \Etgti$, and the test loss of SAM optimal solution on training set $\vtheta^*_{\text{SAM}}, \mathcal{E}(\vtheta^*_{\text{SAM}};\domainT^{(i)})$, has the following bound with probability at least $1-\delta$:
	
	\begin{equation}
	\begin{aligned}
		\mathcal{E}(\vtheta^*_{\text{SAM}};\domainT^{(i)})- & \min_{\vtheta} \Etgti  \leq \hat{\mathcal{E}}_{\text{SAM}}(\vtheta^*_{\text{SAM}}; \mcD) - \min_{\vtheta} \hat{\mathcal{E}}(\vtheta; \mcD) + \Div(\domainD, \domainT^{(i)}) \\
		& \quad \quad \quad \quad \quad \quad \quad \quad +\max_k \consta  + \constb. \\
	\end{aligned}
	\end{equation}	
\end{lemma}

\begin{proof}
We first prove 
\begin{equation}
-\min_{\vtheta} \Etgti  \leq - \min_{\vtheta} \hat{\mathcal{E}}(\vtheta;\mcD) + \frac12\Div (\domainD, \domainT) + \constb
\end{equation}
holds with probability at least $1-\delta$, then we combine it with \cref{lemma1} to prove \cref{lemma2}.

Let $\bar{\vtheta} \in \arg \min_{\vtheta \in \Theta} \Etgti$. From generalisation error bound of $\mathcal{E}(\bar{\vtheta};\domainD)$, the following inequality holds with probability at least $1-\delta$,
\begin{equation}\label{eq:theorem2helper1}
	\hat{\mathcal{E}}(\bar{\vtheta};\mcD)-\mathcal{E}(\bar{\vtheta}; \domainD) \leq \constb,
\end{equation}
where $v$ is a VC dimension of $\Theta$. 

Then, use \cref{eq:theorem2helper1} and \cref{eq:domaindt}, with probability at least $1-\delta$, we have
\begin{equation}\label{eq:theorem2helper2}
\begin{aligned}
	\min_{\vtheta} \hat{\mathcal{E}}(\vtheta; \mcD) & \leq \hat{\mathcal{E}}(\vtheta; \mcD) \\
	\min_{\vtheta} \hat{\mathcal{E}}(\vtheta; \mcD) & \leq \mathcal{E}(\bar{\vtheta}; \domainD) + \constb \\
	& \leq \mcE(\bar{\vtheta}; \domainT^{(i)}) + \frac12\Div(\domainD, \domainT^{(i)}) + \constb \\
	& \leq \min_{\vtheta} \Etgti  + \frac12\Div(\domainD, \domainT^{(i)}) + \constb \\
-\min_{\vtheta} \Etgti  &\leq - \min_{\vtheta} \hat{\mathcal{E}}(\vtheta;\mcD) + \frac12\Div (\domainD, \domainT^{(i)}) + \constb.
\end{aligned}
\end{equation}

Combine \cref{eq:theorem2helper2} with \cref{lemma1}, then with probability at least $1-\delta$ we have 
\begin{equation}
\begin{aligned}
-\min_{\vtheta} \Etgti  &  \leq - \min_{\vtheta} \hat{\mathcal{E}}(\vtheta; \mcD) + \frac12\Div (\domainD, \domainT^{(i)}) + \constb \\
	\mathcal{E}(\vtheta^*_{\text{SAM}};\domainT^{(i)})- \min_{\vtheta}\Etgti & \leq  \hat{\mcE}_{\text{SAM}}(\vtheta^*_{\text{SAM}}; \mcD) + \frac12\Div(\domainD, \domainT^{(i)}) + \max_k\consta  \\ & \quad \quad - \min_{\vtheta} \hat{\mathcal{E}}(\vtheta;\mcD) +  \frac12\Div (\domainD, \domainT^{(i)}) + \constb \\
	\mathcal{E}(\vtheta^*_{\text{SAM}};\domainT^{(i)})- \min_{\vtheta} \Etgti & \leq  \hat{\mcE}_{\text{SAM}}(\vtheta^*_{\text{SAM}}; \mcD)  -\min_{\vtheta} \hat{\mathcal{E}}(\vtheta;\mcD) +  \Div (\domainD, \domainT^{(i)}) \\   
	& \quad \quad \quad \quad  +  \max_k \consta +   \constb.
\end{aligned}
\end{equation}
\end{proof}

\subsection{Proof of \cref{thm:main-theorem}}
We prove \cref{thm:main-theorem} using \cref{lemma2} in this subsection.
\begin{proof}
If we take expectation with respect to $\domainT^{(i)} \sim \mathbb{P}({\domainT})$ on both sides of \cref{lemma2}, we have
\begin{equation}
\begin{aligned}
	& \mathbb{E}_{\domainT^{(i)} \sim \mathbb{P}({\domainT})} \left[\mathcal{E}(\vtheta^*_{\text{SAM}};\domainT^{(i)})- \min_{\vtheta} \mathcal{E}(\vtheta;\domainT^{(i)})\right]  \leq \hat{\mathcal{E}}_{\text{SAM}}(\vtheta^*_{\text{SAM}};\mcD) - \min_{\vtheta} \hat{\mathcal{E}}(\vtheta;\mcD) +  \mathbb{E}_{\domainT^{(i)} \sim \mathbb{P}({\domainT})} \left[\Div (\domainD, \domainT^{(i)}) \right] \\   
	& \qquad \qquad \qquad \qquad \qquad \qquad \qquad \qquad \qquad \quad +  \max_k \consta +   \constb \\
	& \mathcal{E}(\vtheta^*_{\text{SAM}};\domainT)-  \min_{\vtheta} \mathcal{E}(\vtheta;\domainT)  \leq \hat{\mathcal{E}}_{\text{SAM}}(\vtheta^*_{\text{SAM}};\mcD) - \min_{\vtheta} \hat{\mathcal{E}}(\vtheta;\mcD) + \expect_{\domainT^{(i)} \sim \domainT}[\Div(\domainD, \domainT^{(i)})] \\
	& \qquad \qquad \qquad \qquad \qquad \qquad \qquad \qquad \qquad \quad +  \max_k \consta +   \constb. \\
\end{aligned}
\end{equation}	
\end{proof}

\section{Additional Experimental Results}\label{app:results}

\subsection{Backbone Selection}\label{sec:backbone_selection}
To investigate whether our backbone selection strategy selects a backbone that is compatible with data sets in unseen domains, we report the performance of each trained backbone and their corresponding PARC score in \cref{table:parc,table:backbone-selection} respectively. %
In \cref{table:backbone-selection}, we observe that the ImageNet-trained backbone gives the best performance on the coloured data sets, while the Omniglot-trained backbone gives the best performance for MNIST, which is consistent with our backbone selection result in \cref{table:parc}.

\begin{table*}[!ht]
	\centering
	\setlength{\tabcolsep}{5pt}
	\scriptsize
	\caption{We evaluate the performance of all backbones in the backbone bank on unseen domains. ImageNet-trained backbone gives the best performance on coloured data sets (Traffic Sign, MSCOCO, MNIST, CIFAR-10 and CIFAR-100) and the Omniglot-trained backbone gives the best performance on the monochrome data set (MNIST).}
	\begin{tabular}{l|cccccccc}
\toprule
{} & \multicolumn{8}{ c }{\textit{backbone trained on:}} \\[0.2em]
{} &    ilsvrc\_2012     &      omniglot       &      aircraft       &     cu\_birds       &        dtd          &     quickdraw       &       fungi         &    vgg\_flower      \\\hline
{\sc traffic\_sign}  & $ \bf86.02{\pm}0.89  $ & $   49.18{\pm}1.28   $ & $   58.54{\pm}1.23   $ & $   62.57{\pm}1.16   $ & $   69.91{\pm}1.17   $ & $   80.85{\pm}1.10   $ & $   63.17{\pm}1.20   $ & $   61.25{\pm}1.20   $ \\
    {\sc mscoco}     & $ \bf57.07{\pm}1.08  $ & $   20.53{\pm}0.87   $ & $   26.38{\pm}1.01   $ & $   28.89{\pm}1.01   $ & $   30.76{\pm}1.01   $ & $   29.09{\pm}1.05   $ & $   35.23{\pm}1.09   $ & $   31.84{\pm}1.01   $ \\
    {\sc mnist}      & $   94.35{\pm}0.56   $ & $ \bf96.84{\pm}0.34  $ & $   86.70{\pm}0.85   $ & $   90.91{\pm}0.77   $ & $   91.30{\pm}0.67   $ & $   96.32{\pm}0.38   $ & $   92.22{\pm}0.69   $ & $   90.29{\pm}0.70   $ \\
   {\sc cifar10}     & $ \bf80.56{\pm}0.71  $ & $   42.73{\pm}0.75   $ & $   47.57{\pm}0.80   $ & $   49.51{\pm}0.82   $ & $   51.88{\pm}0.83   $ & $   54.43{\pm}0.89   $ & $   53.86{\pm}0.92   $ & $   51.59{\pm}0.86   $ \\
   {\sc cifar100}    & $ \bf72.32{\pm}0.94  $ & $   25.47{\pm}1.01   $ & $   31.88{\pm}1.13   $ & $   37.44{\pm}1.17   $ & $   38.16{\pm}1.19   $ & $   37.27{\pm}1.22   $ & $   45.24{\pm}1.24   $ & $   40.29{\pm}1.18   $ \\
\midrule
   Average Unseen    & $78.07$ & $46.95$ & $50.21$ & $53.87$ & $56.40$ & $59.59$ & $57.94$ & $55.05$ \\
\bottomrule\end{tabular}

	\label{table:backbone-selection}
\end{table*}

\begin{table*}[!ht]
	\centering
	\setlength{\tabcolsep}{5pt}
	\scriptsize
	\caption{PARC scores on unseen domains in the Meta-Dataset for backbones in the backbone bank. Higher score means that the backbone is more compatible to the unseen domain.}
	\begin{tabular}{l|cccccccc}
	\toprule
	{} & \multicolumn{8}{ c }{\textit{backbone trained on:}} \\[0.2em]
	{} & ilsvrc\_2012 & omniglot & aircraft & cu\_birds & dtd & quickdraw & fungi & vgg\_flower \\\hline
	{\sc traffic\_sign} & $\bf 22.22$ & $8.22$ & $19.57$ & $19.91$ & $20.92$ & $19.90$ & $17.91$ & $17.66$ \\
	{\sc mscoco}        & $\bf18.88$ & $7.31$ & $10.80$ & $11.15$ & $10.87$ & $11.84$ & $11.22$ & $12.41$ \\
	{\sc mnist}         & $34.14$ & $\bf45.05$ & $27.65$ & $27.00$ & $30.67$ & $42.44$ & $27.32$ & $26.71$ \\
	{\sc cifar10}       & $\bf29.56$ & $7.75$ & $14.76$ & $12.17$ & $15.91$ & $18.75$ & $13.04$ & $15.55$ \\
	{\sc cifar100}      & $\bf18.26$ & $4.80$ & $7.88$  & $9.31$  & $8.88$  & $10.37$ & $11.13$ & $11.31$ \\
	\bottomrule
\end{tabular}

	\label{table:parc}
\end{table*}

\subsection{Vision Transformer Experiment}
To test whether our proposed training procedure works for different model architecture, we conduct experiment on Vision Transformer (ViT) in this section. We use the ViT-small pre-trained with DINO \cite{oquab2023dino} on ImageNet and fine-tune the pre-trained backbones on MetaDataset with SAM and ERM respectively. After training, we evaluate the performance by using Prototype classifier and the results are given in \cref{table:vit}. Compared with ERM, SAM-trained backbones have better generalisation ability on all data sets.

\begin{table*}[!ht]
	\centering
	\setlength{\tabcolsep}{5pt}
	\scriptsize
	\caption{We fine-tune the pre-trained ViT with SAM and ERM respectively on MetaDataset and evaluate the trained backbone directly with Prototype classifier. SAM-trained backbones result in better generalisation than ERM.}
	\begin{tabular}{l|cc}
\toprule
{} &        SAM          &        ERM          \\\hline
 {\sc ilsvrc\_2012}  & $ \bf63.97{\pm}0.98  $ & $   62.65{\pm}0.96   $ \\
   {\sc omniglot}    & $ \bf86.86{\pm}0.77  $ & $   86.32{\pm}0.81   $ \\
   {\sc aircraft}    & $ \bf84.20{\pm}0.59  $ & $   78.35{\pm}0.68   $ \\
  {\sc cu\_birds}    & $ \bf78.08{\pm}0.76  $ & $   73.50{\pm}0.83   $ \\
     {\sc dtd}       & $ \bf78.56{\pm}0.61  $ & $ \bf78.56{\pm}0.61  $ \\
  {\sc quickdraw}    & $ \bf84.13{\pm}0.53  $ & $   83.70{\pm}0.55   $ \\
    {\sc fungi}      & $ \bf72.46{\pm}0.91  $ & $   68.45{\pm}0.94   $ \\
 {\sc vgg\_flower}   & $ \bf95.20{\pm}0.32  $ & $   93.81{\pm}0.36   $ \\
\midrule
{\sc traffic\_sign}  & $ \bf54.38{\pm}1.05  $ & $   50.46{\pm}1.09   $ \\
    {\sc mscoco}     & $ \bf64.74{\pm}0.88  $ & $   64.33{\pm}0.87   $ \\
    {\sc mnist}      & $ \bf91.94{\pm}0.42  $ & $   90.72{\pm}0.47   $ \\
   {\sc cifar10}     & $ \bf89.27{\pm}0.46  $ & $   88.49{\pm}0.45   $ \\
   {\sc cifar100}    & $ \bf82.85{\pm}0.68  $ & $   82.10{\pm}0.68   $ \\
\midrule
    Average Seen     & $80.43$ & $78.17$ \\
   Average Unseen    & $76.64$ & $75.22$ \\
\bottomrule\end{tabular}

	\label{table:vit}
\end{table*}

\section{Implementation Details}\label{app:implementation}

For a fair comparison with prior work, we use RestNet-18 as the backbone in all experiments. 
We use NVIDIA V100 GPUs for backbone training and run the experiments on a cluster. 

\subsection{Implementation Details for \cref{table:flatness}}
For ERM, we use the backbone provided by SUR where the SGD optimiser is used for all backbones.
To eliminate the influence introduced by different optimiser, we adopt SGD as well for SAM backbone training.
For b-SAM, we follow the optimisation algorithm proposed in the paper.
For both SAM and b-SAM, we use cosine annealing learning rate decay with a restart and provide their hyperparameters in \cref{tab:sam-hyperparameters} and \cref{tab:bsam-hyperparameters} respectively.

\begin{table}[ht]
\centering
\scriptsize
\caption{Hyperparameters of SAM training in \cref{table:flatness}.}
\begin{tabular}{l|ccccc}
	\toprule
	{} & {Batch Size} & {Learning Rate} & {Total Iterations} & {Optimizer Restart} & $\rho$\\\hline
	\sc{ilsvrc\_2012} & 128 & 0.01 & 480000 & 48000 & 0.05 \\
	\sc{omniglot} & 16 & 0.03& 50000 & 3000 & 0.01 \\
	\sc{aircraft} & 8 & 0.03 & 50000 & 3000 & 0.1 \\
	\sc{cu\_birds} & 16 & 0.03 & 50000 & 3000 & 0.1 \\
	\sc{dtd} & 32 & 0.03 & 50000 & 1500 & 0.1 \\
	\sc{quickdraw} & 64 & 0.01 & 480000 & 48000 & 0.05 \\
	\sc{fungi} & 32 & 0.03 & 480000 & 15000 & 0.05 \\
	\sc{vgg\_flower} & 8 & 0.03 & 50000 & 1500 & 0.1 \\
	\bottomrule
\end{tabular}
\label{tab:sam-hyperparameters}
\end{table}

\begin{table}[ht]
    \centering
	\caption{Hyperparameters of b-SAM training in \cref{table:flatness}. For all data sets, we set $\rho=0.01$, $\gamma=0.1$, and $\beta_1=0.9$.}
	\scriptsize
\begin{tabular}{l|ccccc}
        \toprule
        {} & {Batch Size} & {Learning Rate} & {Prior Precision} & {Total Iterations} & {Optimiser Restart}\\
        \midrule
        \sc{ilsvrc\_2012} & 500 & 0.1 & 100 & 250000 & 50000\\
        \sc{omniglot} & 200 & 0.1 & 100 & 10000 & 2000\\
        \sc{aircraft} & 200 & 0.5 & 10 & 5000 &1000\\
        \sc{cu\_birds} & 200 & 0.5 & 10 & 5000 &1000\\
        \sc{dtd} & 200 & 0.1 & 10 & 5000 &1000\\
        \sc{quickdraw} & 200 & 0.5 & 50 & 30000 &6000\\
        \sc{fungi} & 200 & 0.1 & 50 & 30000 & 6000\\
        \sc{vgg\_flower} & 200 & 0.1 & 10 & 5000 & 1000\\
        \bottomrule
    \end{tabular}
    \label{tab:bsam-hyperparameters}
\end{table}

\subsection{Implementation Details for \cref{table:fine-tune}}
We use SGD optimisers and cosine annealing learning rate decay with a restart for both vanilla fine-tuning and LoRA fine-tuning. 
For LoRA, we set the rank $r=10$ for all data sets. 
The hyperparameters for vanilla and LoRA fine-tuning are provided in \cref{tab:vanilla-finetune-hyper} and \cref{tab:lora-finetune-hyper} respectively.

\begin{table}[ht]
    \centering
	\caption{Hyperparameters for vanilla fine-tuning in \cref{table:fine-tune}.}
	\scriptsize
\begin{tabular}{l|cccc}
        \toprule
        {} & {Batch Size} & {Learning Rate} & {Iterations} & {Optimizer Restart} \\\hline
        \sc{omniglot} & 16 & 0.01 & 10000 & 2000 \\
        \sc{aircraft} & 8 & 0.001 & 10000 & 2000 \\
        \sc{cu\_birds} & 16 & 0.001 & 10000 & 2000 \\
        \sc{dtd} & 32 & 0.001 & 10000 & 1000 \\
        \sc{quickdraw} & 64 & 0.005 & 100000 & 5000 \\
        \sc{fungi} & 32 & 0.01 & 80000 & 5000 \\
        \sc{vgg\_flower} & 8 & 0.001 & 10000 & 2000 \\
        \bottomrule
    \end{tabular}
    \label{tab:vanilla-finetune-hyper}
\end{table}

\begin{table}[ht]
    \centering
	\caption{Hyperparameters for LoRA fine-tuning in \cref{table:fine-tune}.}
	\scriptsize
\begin{tabular}{l|cccc}
        \toprule
        {} & {Batch Size} & {Learning Rate} & {Iterations} & {Optimizer Restart} \\\hline
        \sc{omniglot} & 16 & 0.005 & 40000 & 2000 \\
        \sc{aircraft} & 8 & 0.005 & 40000 & 2000 \\
        \sc{cu\_birds} & 16 & 0.005 & 10000 & 2000 \\
        \sc{dtd} & 32 & 0.001 & 10000 & 2000 \\
        \sc{quickdraw} & 64 & 0.01 & 100000 & 5000 \\
        \sc{fungi} & 32 & 0.005 & 80000 & 5000 \\
        \sc{vgg\_flower} & 8 & 0.001 & 10000 & 2000 \\
        \bottomrule
    \end{tabular}
    \label{tab:lora-finetune-hyper}
\end{table}

\subsection{Implementation Details for \cref{table:compare-with-sota}}
We use SGD optimisers and cosine annealing learning rate decay with a restart for SAM objective fine-tuning. 
The hyperparameters are given in \cref{tab:samtf-hyper} and we set $\rho=0.05$ for all data sets.
When combining SAM with knowledge distillation, we use the provided hyperparameter setting for backbone training in URL.

\begin{table}[ht]
    \centering
	\caption{Hyperparameters for SAM combined with vanilla fine-tuning in \cref{table:compare-with-sota}.}
	\scriptsize
\begin{tabular}{l|cccc}
        \toprule
        {} & {Batch Size} & {Learning Rate} & {Iterations} & {Optimizer Restart} \\\hline
        \sc{omniglot} & 16 & 0.01 & 20000 & 2000 \\
        \sc{aircraft} & 8 & 0.001 & 20000 & 2000 \\
        \sc{cu\_birds} & 16 & 0.001 & 20000 & 2000 \\
        \sc{dtd} & 32 & 0.001 & 20000 & 1000 \\
        \sc{quickdraw} & 64 & 0.01 & 100000 & 5000 \\
        \sc{fungi} & 32 & 0.01 & 80000 & 5000 \\
        \sc{vgg\_flower} & 8 & 0.001 & 10000 & 2000 \\
        \bottomrule
    \end{tabular}
    \label{tab:samtf-hyper}
\end{table}

\end{document}